\documentclass{ifacconf}

\usepackage{amsmath}
\usepackage{amssymb}
\usepackage{mathtools}

\usepackage{graphicx}

\usepackage[shortlabels]{enumitem}

\usepackage{array}
\usepackage{booktabs}
\usepackage[hyphens]{url}
\usepackage[dvipsnames]{xcolor}
\usepackage{bm}
\usepackage{multirow}
\usepackage{pifont}

\newcommand{\cmark}{\textcolor{green!60!black}{\ding{51}}}
\newcommand{\xmark}{\textcolor{red!70!black}{\ding{55}}}

\newtheorem{dfn}{Definition}
\newtheorem{lmm}{Lemma}

\newenvironment{proof}{\begin{pf}}{\end{pf}}

\makeatletter
\let\cl@part\relax
\makeatother
\usepackage[capitalize,noabbrev,nameinlink]{cleveref}

\newcommand{\reals}{\mathbb{R}}

\newcommand{\simplehref}[2]{%
  \pdfstartlink attr{/Border [0 0 0]} user{/Subtype /Link /A << /S /URI /URI (#1) >>}#2\pdfendlink%
}

\makeatletter
\providecommand{\doi}[1]{}
\providecommand{\url}[1]{}

\renewcommand{\doi}[1]{\@ifnextchar.{\@gobble}{}}
\renewcommand{\url}[1]{\@ifnextchar.{\@gobble}{}}

\makeatother

\begin{document}

\begin{frontmatter}

\title{Compliant Explicit Reference Governor for Contact Friendly Robotic Manipulators}
\vspace{-10pt}
\thanks[corresponding]{Equal Contribution}
\thanks[footnoteinfo1]{$~$Supported by NSF-CMMI CAREER Grant \#2046212.}
\thanks[footnoteinfo2]{$~~$Supported by NASA Space Technology Graduate Research Opportunity Grant 80NSSC22K1211.}

\author{Yaashia Gautam\thanksref{corresponding} \thanksref{footnoteinfo1}},
\author{Gilberto Briscoe-Martinez\thanksref{corresponding} \thanksref{footnoteinfo2}},
\author{Adhitya Mohan},
\author{Nataliya  Nechyporenko},
\author{Alessandro Roncone},
\author{Marco M. Nicotra}
\vspace{-10pt}

\address{College of Engineering and Applied Sciences, University of Colorado Boulder,
   Boulder, CO 80301 USA \\(e-mail: firstname.lastname@colorado.edu)}

\begin{abstract}
 This paper introduces the Compliant Explicit Reference Governor (CERG), a modular reference management system that enables robots to interact physically with their environment under provable guarantees. The CERG is an intermediate layer that can be placed between a high-level planner and a low-level controller: it enforces operational constraints and enables smooth transitions between free-motion and contact operations.  The CERG ensures safety by limiting the total energy available to the robotic arm at the time of contact. In the absence of contact, however, the CERG does not penalize the system performance.
 Simulation and hardware experiments validate the CERG on increasingly complex systems. \vspace{-10pt}

\end{abstract}
\vspace{-10pt}
\begin{keyword}
 Constrained control, task and motion planning, robotic grasping and manipulation, reference governors, real time motion control
\end{keyword}

\end{frontmatter}

\section{Introduction}\vspace{-10pt}
While traditional motion planning prioritizes collision-free trajectories
to guarantee safety and avoid self-damage, \cite{elbanhawi2014sampling},
modern robotics increasingly requires contact-aware interaction.
Indeed, physical interaction with the environment, rather than its avoidance, paves the way to more human-like behaviors where
contact with the environment is treated as a resource rather than a hindrance,
\cite{PushingPaper}.\smallskip \vspace{-5pt}

To achieve contact-rich capabilities,
robots must transition between free motion and contact safely and coherently.
Moreover, both the robot and the environment impose limits on the feasible motion and shape how physical interaction can occur. We distinguish between hard constraints (joint limits, torque bounds, forbidden surfaces) and soft constraints, where contact is permissible if handled safely.
These soft constraints translate into highly non-convex OR conditions, where the robot must either a) avoid contact entirely, or b) operate under contact-safe restrictions. \vspace{-5pt}

Optimization-based strategies for contact-rich manipulation and trajectory optimization(\cite{adaptive_ci_mpc,Fast_CI_MPC,GlobalContact}) rely on a high-level Model Predictive Control (MPC) layer that relies on
a time-varying local approximation of the contact dynamics. However, they require a full redesign of the control law, which is unappealing for practitioners with a legacy controller whose behaviors are easily interpretable. \vspace{-5pt}

Conversely, hierarchical approaches can add constraint handling to an existing controller without modifying it.
 \cite{Energy_Passive_control} ensures safety by manipulating the compliance of the controller using energy-based considerations. \cite{cbf_contact} filters the control input using  a control barrier function, but the former lacks constraint guarantees and the latter, asymptotic stability guarantees.

The Explicit Reference Governor (ERG) is a
hierarchical safety filter that manipulates the \emph{reference} of an existing controller, rather than modifying it and provides strict safety and stability guarantees \cite{ERG}. \cite{ERG_Robot,RRT+ERG} applied the ERG to robotic arms, but only in the context of contact avoidance.
In this paper, we introduce the \textsl{Compliant Explicit Reference Governor} (CERG): a constraint-handling safety filter tailored to contact-rich tasks. \vspace{-5pt}

Selecting a safety metric for contact operations is a central challenge.
Certain methods (\cite{dickson2025safe, li2021provably}) use interaction force as a safety guideline. Following \cite{Energy_Passive_control}, we define interaction safety using a bound on the maximum total energy (potential+kinetic) of the closed-loop system. Energy-based safety offers several advantages: it is frame-invariant, ties directly to stability guarantees, and is easily interpretable. \vspace{-5pt}

Thus, the CERG enforces safety by allowing the robot to either avoid contact during free motion or engage in contact with bounded mechanical energy. This framework can seamlessly transition between modes without modifying the controller or requiring a specialized solver for the disjunctive constraints. The main contributions are: \vspace{-5pt}

\begin{itemize}
\item A modular reference filtering framework that seamlessly transitions between contact-free and contact-rich motion, while respecting joint angle, velocity, and torque constraints;
\item A safety mechanism for disjunctive (OR-type) constraints, ensuring provable constraint satisfaction without controller modification;
\item An energy-based safety guideline that supports provably safe interaction.
\end{itemize}

The CERG is validated both numerically (on MATLAB and Drake) and experimentally on a Franka Emika Panda.

\vspace{-7pt}
\subsection{Notation}\vspace{-10pt} Given two column vectors $a$ and $b$, the column-wise vector concatenation is denoted as $[a;b]=[a^\top~b^\top]^\top$.\vspace{-5pt}

\section{Problem Statement}\vspace{-5pt}

Consider the dynamic model of a fully actuated robotic arm subject to the Euler Lagrange equation
\begin{equation} \label{cts sys}
    M(q) \ddot q + C(q,\dot q)\dot q + g(q) = u,
\end{equation}
where $q \in \reals^n$ are the degrees of freedom, $u \in\reals^n$ are the actuator inputs, and $M$, $C$ and $g$ are the mass matrix, Coriolis matrix, and gravity vector as defined in \cite{ortega1998euler}, with $\dot M-2C$ being skew symmetric. We use $x = [q^\top;\dot q^\top]$ to denote the full state. The system is subject to \emph{hard} state-and-input constraints
\begin{equation}\label{eq:hard}
    h(q,\dot q,u)\leq0,
\end{equation}
which capture a wide range of restrictions (e.g. joint range limitations, joint velocity limits, actuator saturation, surfaces that the robot end effector cannot come into contact with). The system is also subject to a \emph{soft} constraint
\begin{equation}\label{eq:soft}
    s(q)<0.
\end{equation}
This defines a surface the robot may contact, but only with a limited amount of energy.
Both $h(q,\dot q,u)$ and $s(q)$ are assumed to be $C^1$ continuous.
The robotic arm is in \emph{free motion} if $s(q)<0$ and is in \emph{contact} if $s(q)\geq0$. \smallskip

Treating $s(q)< 0$ as a hard constraint is often overly restrictive, since contact is frequently desirable if safe.
We develop a constrained control strategy that transitions seamlessly between free motion and contact via compliance, satisfying:
\begin{itemize}
    \item Enforce $h(q,\dot q,u)\leq0$ at all times;
    \item Enforce $s(q)<0$ \textbf{-OR-} allow safe contact;
    \item If possible, steer the arm to a reference $r\in\reals ^n$;
    \item Run in real-time with a low computational footprint.
\end{itemize}
Building on prior work on provably safe real-time constrained control of robotic arms \cite{ERG_Robot,RRT+ERG}, this paper addresses contact by leveraging the ability of the ERG framework to handle ``OR'' constraints.
\vspace{-15pt}

\section{Preliminaries}\vspace{-5pt}
The Explicit Reference Governor (ERG) is an optimisation-free reference management scheme that divides the control problem into a \emph{\textbf{prestabilisation}} and a \emph{\textbf{reference management}} subproblem.
Prior ERG implementations (\cite{ERG_Robot,RRT+ERG}) treat contact as undesirable: if an external agent initiates contact, the robot stops until the agent disengages.
This section presents the ERG framework and identifies modifications needed for voluntary contact.
\vspace{-8pt}
\subsection{Prestabilisation}\label{ssec: Prestab}\vspace{-5pt}
The purpose of the prestabilization layer is to design a constraint-agnostic controller that steers the robot to a joint configuration $v\in\reals ^n$. This paper considers a
PD with gravity compensation in joint space, \cite{Arimoto1996}
\begin{equation}\label{PD control}
    u = -K_P(q-v) - K_D\dot q + g(q),
\end{equation}
Given a constant $v\in\reals^n$, the equilibrium point $\bar x_v=[v;0]$ is globally asymptotically stable, as shown via LaSalle's Invariance Principle with
the Lyapunov Function
\begin{equation}\label{lyapunov}
    V(x, v) = \frac{1}{2} \dot q^\top M(q) \dot q + \frac{1}{2} (q-v)^\top K_P (q-v),
\end{equation}
which satisfies $\dot V(x,v) = - \dot q^\top K_D \dot q.$
Note that \eqref{lyapunov} represents the total energy of the prestabilized system: \emph{real} kinetic energy stored in the robot and the \emph{virtual} potential energy stored in the control law, and is monotone decreasing whenever $v$ is constant.

\vspace{-10pt}
\subsection{Explicit Reference Governor}\vspace{-5pt}

Given a target $r\in\reals^n$, the ERG generates a time-dependent signal $v:[0,\infty)\to\reals^n$ so that the prestabilized system \eqref{cts sys}, \eqref{PD control} ensures: (i) constraints \eqref{eq:hard} are satisfied at all times, and (ii) if possible, $\lim_{t\to\infty}q(t)=r$.
This is achieved by manipulating the derivative of the reference according to
\begin{equation} \label{aux ref}
    \dot v = \Delta(x,v) \rho(v,r)
\end{equation}
where $\rho(v,r)$ is the \textbf{Navigation Field}, which solves the kinematic problem \emph{``what steady-state admissible path can the robot take to reach the reference?''}, and $\Delta(x,v)$ is the \textbf{Dynamic Safety Margin}, which solves the dynamic problem \emph{``how fast can the robot travel along its path without violating constraints?''}. The ERG framework, including safety and stability proofs, is detailed in  \cite{ERG}.
\vspace{-5pt}
\subsubsection{Navigation Field (NF):} This paper considers the basic attraction/repulsion NF
\begin{equation}\label{NF_plain}
    \rho(v,r) = \rho_{att}(v,r)+\rho_h(v),
\end{equation}
where
\begin{equation}\label{att field}
    \rho_{att}(r,v) =\frac{r-v}{\mathrm{max} (\| r-v\| , \eta)}
\end{equation}
is a conservative vector field that points from $v$ to $r$,
\begin{equation}\label{eq: rho_h}
    \rho_h(v) = \mathrm{max}  \left( \frac{\zeta + h_{ss}(v)}{\zeta - \delta_h}, 0 \right) \vec{\rho}_h(v),
\end{equation}
is a conservative vector field that points to the interior of the steady-state constraint $h_{ss}(v) = h(v,0,g(v))$, and $\eta>0$, $\zeta>\delta_h>0$ are positive constants.
Computation of $\vec\rho_h(v)$ follows \cite{ERG_Robot}. Local minima can be avoided via RRT augmentation \cite{RRT+ERG}, which is outside this paper's scope.

\subsubsection{Dynamic Safety Margin (DSM)} For simplicity, this paper only considers the trajectory-based DSM
\begin{equation}\label{DSM_plain}
    \Delta_h(x,v) = \kappa_h \inf_{\tau\geq0}(-\hat h(\tau|x,v)),
\end{equation}
where $\kappa_h>0$, and the constraint prediction
\begin{equation}
    \hat h(\tau|x,v) = h(\hat q(\tau),\dot{\hat q}(\tau),\hat u(\tau)),
\end{equation}
is obtained by solving the prestabilized system
\begin{equation}\label{eq: TrajPred}
\begin{cases}
    M(\hat q)\ddot{\hat q}+C(\hat q,\dot {\hat q})\dot {\hat q}+g(\hat q)=\hat u,\\
    \hat u= -K_P (\hat q - v) - K_D\dot {\hat q} + g(\hat q)
\end{cases}
\end{equation}
with initial conditions $[\hat q(0)^\top;\dot{\hat q}(0)^\top]=x$. The forward dynamics are integrated via Symplectic Euler for computational efficiency and details on finite-horizon predictions and discrete dynamics are in \cite{ERG_Robot}.

If $s(q)$ is treated like a hard constraint, the proposed ERG scheme outputs $\dot v=0$ and brings the robot to a full stop. Although this behavior is safe, it prevents movement until contact is resolved. The next section introduces the main contribution: a compliant ERG enabling the robot to operate safely ($\dot v\neq0$) while in contact.

\section{Compliant Explicit Reference Governor}\vspace{-5pt}
 \begin{figure}
 \vspace{10pt}
    \centering
    \includegraphics[width=\columnwidth]{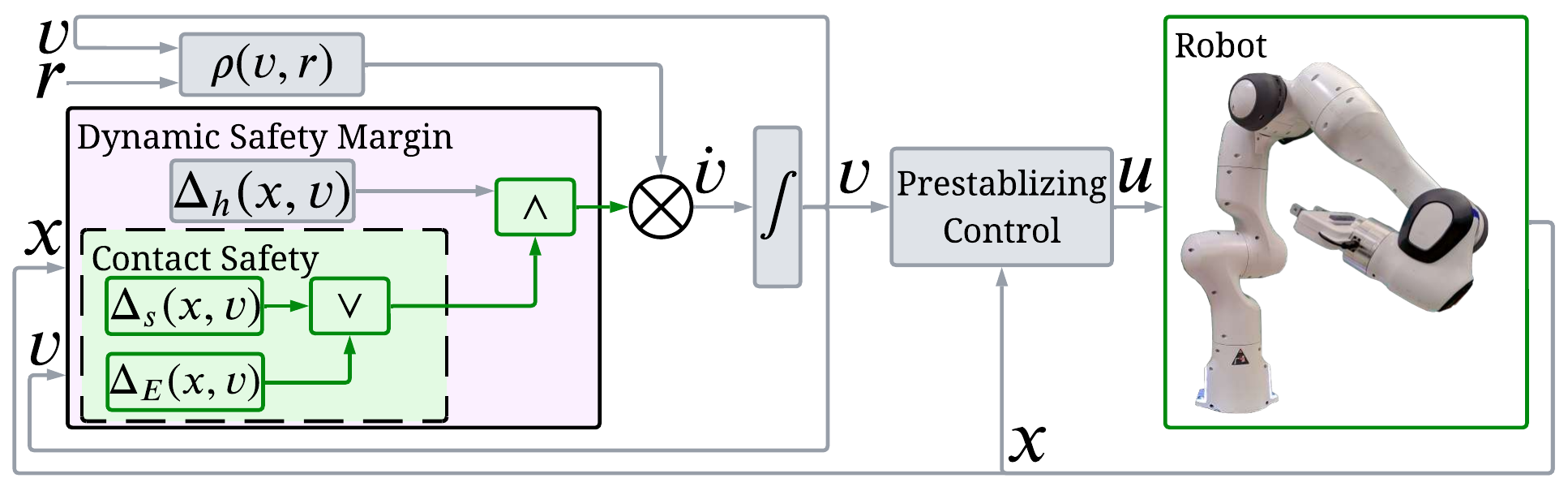}
    \caption{CERG decomposes the Dynamic Safety Margin into Hard and Soft terms. The Soft DSM $\Delta_s(x,v)$ and Energy DSM $\Delta_E(x,v)$ ensure safe contact, while a Soft Navigation Field limits penetration. The applied reference integrates $\rho(v,r)$ scaled by the DSM.\vspace{-5pt}
    \label{CERG_flow}}
\end{figure}
The CERG architecture is illustrated in Fig. \ref{CERG_flow}.
In analogy to \cite{Energy_Passive_control,cbf_contact}, we consider it ``safe'' to come into contact with the robot if
\begin{equation}\label{eq:E_cst}
    V(x,v)\leq E_{\max},
\end{equation}
where $E_{\max}>0$ is an upper bound on the total (kinetic + potential) energy available to the robot. The aim is to manipulate $\dot v(t)$ so that
\begin{itemize}
    \item\emph{Free Motion:} the total energy is unconstrained, as limiting the kinetic energy would unnecessarily slow the robot;
    \item \emph{Contact Operations:} the total energy stored in the system is less than a given safety threshold $E_{\max}>0$.
\end{itemize}
These requirements can be expressed as a logic OR constraint
\begin{equation}\label{eq: OR constraint}
s(q(t))<0~~\vee~~V(x(t),v(t))\leq E_{\max},\quad \forall t\geq0,
\end{equation}
which is equivalent \cite{OR_constraints} to
\begin{equation}\label{eq: compliant_contact}
    \min(s(q),V(x,v)-E_{\max})\leq0.
\end{equation}
To also account for hard constraints $h(q,\dot q,u)\leq0$, we must satisfy $c(x(t),u(t),v(t))\leq0,~\forall t\geq0$, with
\begin{equation} \label{Energy_constraints}
       c(x,u,v) =  [h(q,\dot q,u); ~    \min (s(q), V(x,v) - E_{\max})].
\end{equation}
The compliant ERG can now be obtained by designing a suitable Dynamic Safety Margin and Navigation Field.
\vspace{-5pt}
\subsection{Compliant Dynamic Safety Margin}\vspace{-5pt}
From \cite[Definition 1]{ERG}:

\begin{dfn}\label{DSM_def}
A Lipschitz continuous function $\Delta (x,v)$ is a ``Dynamic Safety Margin'' for the constraint $c(x,u,v)\!\leq\!0$ if
\begin{itemize}
    \item  $\Delta(x,v) > 0 \implies c(\hat x(\tau),\hat u(\tau), v) < 0,~\forall\tau\geq 0$;
    \item $\Delta(x,v) \geq 0 \implies c(\hat x(\tau),\hat u (\tau), v) \leq 0,~\forall\tau\geq 0$;
    \item $\Delta(x,v) = 0 \implies \Delta(\hat x(\tau), v) \geq 0 ,\qquad~\forall\tau\geq 0$;
    \item $\forall \delta>0,~\exists \epsilon>0:~c(\bar x_v,\bar u_v,v)\leq-\delta\Rightarrow\Delta(\bar x_v,v)\geq\epsilon$;
\end{itemize}
where $\hat x(\tau)=[\hat q(\tau);\dot{\hat q}(\tau)]$ and $\hat u(\tau)$ are the solutions to \eqref{eq: TrajPred}.
\end{dfn}
The following Lemma states that the ERG law \eqref{aux ref} ensures forward invariance of the hard constraints.

\begin{lmm}\label{lmm1}
    The DSM $\Delta_h(x,v)$ satisfies the property \\ $\Delta_h(x(t),v) \geq \Delta_h(x(0),v),~\forall t \geq 0$.
\end{lmm}
\begin{proof}
Given a constant $v$, the trajectory predictions match the prestabilized dynamics \eqref{eq: TrajPred} . As a result,
\begin{equation}
    \inf_{\tau\geq0}(-\hat h(\tau|x(t),v))=\inf_{\tau\geq t}(-\hat h(\tau|x(0),v)),
\end{equation}
which satisfies
\begin{equation}
\inf_{\tau\geq t}(-\hat h(\tau|x(0),v))\geq\inf_{\tau\geq 0}(-\hat h(\tau|x(0),v)).
\end{equation}
The statement then follows from \eqref{DSM_plain}.
\end{proof}
We now define the DSM for soft constraints, $ \Delta_s(x,v)$, and prove an important property
\begin{prop}\label{cor1}
Let
\begin{equation}\label{DSM_soft1}
    \Delta_s(x,v)=\kappa_s \inf_{\tau\geq0}(-\hat s(\tau|x,v)),
\end{equation}
where $\kappa_s>0$, $\hat s(\tau|x,v) = s(\hat q(\tau))$, and $\hat q(\tau)$ is the solution to \eqref{eq: TrajPred}. Then,
\begin{equation}
    \Delta_s(x(t),v) \geq \Delta_s(x(0),v),~\forall t\geq 0.
\end{equation}
\end{prop}
\begin{proof}
    The property holds since the position constraint $s(q)\leq0$ is a special case of the general state-and input constraint $h(q,\dot q,u)\leq0$.
\end{proof}

We now provide a DSM for the compliant constraint.
\begin{thm}\label{thm: DSM}
    Let $\Delta_h$, $\Delta_s$ be defined as in \eqref{DSM_plain}, \eqref{DSM_soft1}, and let
    \begin{equation}
        \Delta_E(x,v)=\kappa_E(E_{\max}-V(x,v)),
    \end{equation}
    with $\kappa_E>0$ and $V(x,v)$ given in \eqref{lyapunov}. Then,
    \begin{equation}
        \Delta(x,v) = \min(\Delta_h,\max(\Delta_s,\Delta_E))
    \end{equation}
    is a Dynamic Safety Margin for \eqref{Energy_constraints}.
\end{thm}
\begin{proof}
    By construction, $\Delta_h\!>\!0$ implies $h(\hat x(\tau),\hat u(\tau))\!<\!0,$ $\forall \tau\geq0$ and $\Delta_s\!>\!0$ implies $s(\hat q(\tau))\!<\!0,$ $\forall \tau\geq0$. The same holds for non-strict inequalities (i.e. ``$\geq$'' instead of ``$>$''). \smallskip

    Lemma \ref{lmm1} and Proposition \ref{cor1} ensure $\Delta_h(\hat x(\tau),v)\geq\Delta_h(x,v)$ and $\Delta_s(\hat x(\tau),v)\geq\Delta_s(x,v)$, for all $\tau\geq0$.\smallskip

    Finally, $h(v,0,g(v))\leq-\delta\implies \Delta_h(x,v)\geq\delta$ and $s(v)\leq-\delta\implies \Delta_s(x,v)\geq\delta$ hold by construction since the equilibrium point $(x,v)=(\bar x_v,v)$ always satisfies $\hat x(\tau|x,v)=\bar x_v,~\forall \tau\geq0$.\medskip

    These properties are all a direct consequence of the fact that $\Delta_h$ and $\Delta_s$ are trajectory-based DSMs for their respective constraints. As for $\Delta_E$, it follows from $\dot V(x,v) = - \dot q^\top K_D \dot q$ that
    \begin{equation}
        V(\hat x(\tau),v)\leq V(x,v),\qquad \forall \tau\geq0.
    \end{equation}
    Thus, $E_{\max}-V(x,v)\geq0$ implies $E_{\max}-V(\hat x(\tau),v)\geq0$, $\forall \tau\geq0$. This is sufficient to show that $\Delta_E(x,v)>0$ implies $V(\hat x(\tau),v)-E_{\max}<0$ (same with non-strict signs) and that $\Delta_E(x,v)\geq0\implies\Delta_E(\hat x(\tau),v)\geq0,~\forall \tau\geq0$. As for the final property of DSMs, we note that, regardless of $v$, the Lyapunov function satisfies $V(\bar x_v,v)=0$. Therefore, the system always satisfies $\Delta_E(\bar x_v,v)\geq\kappa_E E_{\max}>0$. This proves that $\Delta_h, \Delta_s$ and $\Delta_E$ satisfy Definition \ref{DSM_def} for their respective constraints.

    To conclude the proof, we now look at how the individual DSMs are composed. First, consider
    \begin{equation}
        \Delta^\prime(x,v)=\max(\Delta_s(x,v),\Delta_E(x,v)).
    \end{equation}
    Since $\Delta^\prime>0$ implies either $\Delta_s>0$ or $\Delta_E>0$, it follows  that $\Delta^\prime(x,v)>0$ is sufficient to satisfy the OR constraint \eqref{eq: OR constraint}. Moreover, it follows from the time monotonicity of $\Delta_s,\Delta_E$ that $\Delta^\prime(\hat x(\tau),v)\geq\Delta^\prime(x,v),~\forall \tau\geq0$. Finally, we note that $\min(s(\hat q_v),V(x,v)-E_{\max})\leq-\delta$ implies $\Delta^\prime(x,v)\geq\max(\delta,\kappa_EE_{\max})$.
    Similar reasoning applies to the overall function $\Delta(x,v)=\min(\Delta_h(x,v),\Delta^\prime(x,v))$, where $\min$ operator enforces an AND rather than an OR constraint.
\end{proof}

Theorem \ref{thm: DSM} yields a Dynamic Safety Margin that mixes hard and soft constraints. The former are enforced at all times, the latter features two alternatives, only one of which needs to be enforced at any given time.

\subsection{Soft Navigation Field}\vspace{-5pt}
Unlike hard constraints, the soft navigation field  should allow $v(t)$ to penetrate the constraint boundary. Having $v(t)$ inside an object is what enables the compliant controller to enter contact with that object. The following ``soft'' repulsion field is proposed
\begin{equation}\label{soft_rep}
    \rho_s(v) = \mathrm{max}   \left( \frac{s(v)}{\delta_s}, 0  \right ) \vec\rho_s(v),
\end{equation}
where the unitary vector $\vec\rho_s(v)$ points to the interior of the constraint and can be computed as detailed in \cite{ERG_Robot}. This particular repulsion field satisfies
\begin{itemize}
    \item $s(v)<0\implies\rho_s(v)=0$: the repulsion field is inactive whenever $v$ is outside the constraint;
    \item $s(v)=\delta_s\implies\|\rho_s(v)\|=1$: preventing $v$ from penetrating beyond $\delta_s$ since the attraction field satisfies the bound $\|\rho_{att}(v,r)\|\leq1,~\forall v\in\reals^n$.
\end{itemize}
The resulting attraction/repulsion field is
\begin{equation}
    \rho(v,r)=\rho_{att}(v,r)+\rho_h(v)+\rho_s(v),
\end{equation}
again augmentable with RRT to avoid local minima.

\section{Numerical Examples}\label{sec: NumEx} \vspace{-5pt}
This section validates CERG framework for increasingly complex systems and contact models, starting with a two-link arm, progressing to a 7-DoF Franka Emika simulation in Drake, and concluding with hardware experiments, on the Franka Emika arm.
We benchmark CERG against ERG and a standard Cartesian impedance controller with a null space control on hardware, demonstrating its significant practical advantages. Videos are available at
\textcolor{blue}{\underline{\texttt{http://yaashia-g.github.io/publications/CERG/}}}.

\subsection{Two Link Planar Arm}\vspace{-5pt}
For the purpose of this simulation, the contact forces are modeled as a unidirectional spring damper system.
\begin{equation}\label{Ext Force}
    F_{ext} = -\max(0, Ks(q))\nabla s(q)-\max(0, Bs(q))\dot q
\end{equation}
where $K$ and $B$ are the spring and damping coefficients of the object.
With $\vec\rho_s(v)=[0~-1]^\top$ in \eqref{soft_rep}, the navigation field steers $v$ towards the solution to
\begin{align}
    \min~&\|v-r\|^2\label{eq:v_opt}\\
    \mathrm{s.t.}~~ & s(v)\leq-\delta_s.\nonumber
\end{align}
This property follows from convexity of the constraints and the attraction/repulsion NF being a regularized projected gradient flow, \cite{ERG}. For this simple example,
since $K_P$ represents a virtual spring between $x$ and $v$, a desired steady-state force $F_{ss}>0$ maps directly to  maximum penetration depth $\delta_s$.
\begin{equation}\label{eq:ss-map}
    \delta_s=\frac{F_{ss}}{K_P}.
\end{equation}
Noting that the corresponding potential energy at steady state is $E_{ss}=\frac12K_P\delta_s^2$, we must satisfy the energy upperbound $E_{ss}<E_{\max}$ if we wish to attain the target steady-state force.

We consider a two-link planar arm (RR arm) with dynamics satisfying \eqref{cts sys}  as in \cite{ModernRobotics}.
The parameters $l_1, l_2, m_1, m_2 $ are 1m, 0.5m, 2kg and 1.5kg respectively.
The end effector position is $p=f(q)$, where $f(q)$ are the forward kinematics. The initial condition is $q_0 = [\pi/2,\, \pi/5]^\top$ and the reference $r = [\pi/4,\, -\pi/3]^\top$, corresponding to end-effector position $[1.19,\, 0.57]^\top$.
The system has soft constraints $p_x \leq 1$, which the reference violates.
Fig. \ref{RR_joint} shows the trajectory of the end effector reference $f(v)$ and the steady-state configuration reached by the RR arm. Given
\begin{equation}
    \vec q_s(v)=\frac{J^\dagger(v)[1~0]^\top}{\|J^\dagger(v)[1~0]^\top\|},
\end{equation}
for \eqref{soft_rep}, where $J(v)$ is the Jacobian of the forward kinematics and $J^\dagger$ is its pseudoinverse, following \cite{ERG_Robot}, $v$ converges to a local minimum of
\begin{align}
    \min~&\|f(v)-f(r)\|^2\\
    \mathrm{s.t.}~~ & s(v)\leq-\delta_s.
\end{align}
Despite this, note that $p$ \emph{does not} converge to the Cartesian projection of $f(v)$ onto the constraints, as CERG is implemented in joint space, instead of end effector space.

A possible way to address this issue is to reformulate the CERG in end effector space using the control law
\begin{equation}
    u = -k_pJ(q)^\top(f(q)-f(v)) - k_d\dot q + g(q).
\end{equation}
The corresponding energy function
\begin{figure}
    \centering
    \includegraphics[trim=.7cm 0.1cm 0.65cm 0cm, clip=true,scale=0.57]{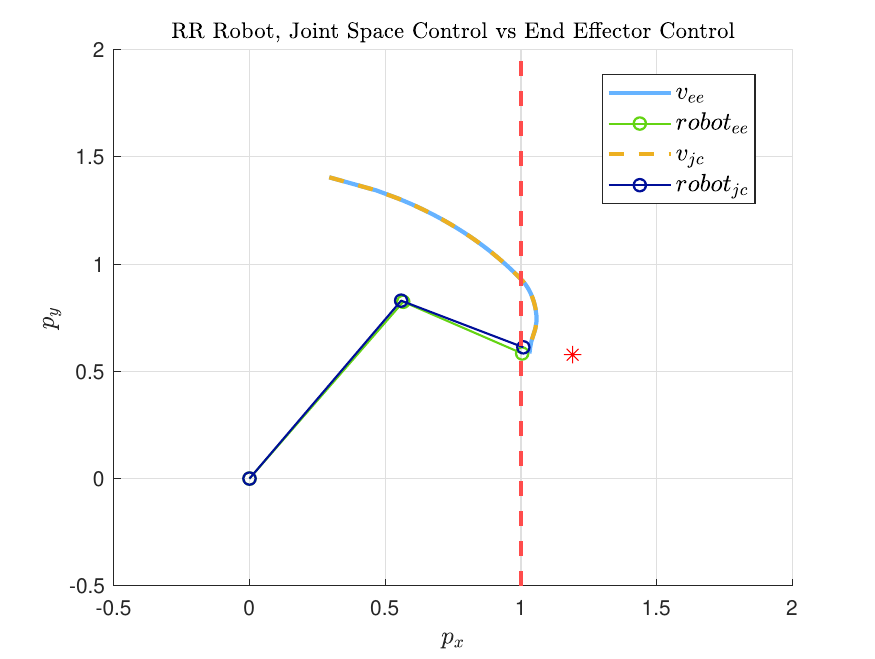}
    \caption{ CERG applied to the RR arm with joint-space and end-effector control. The final configurations minimize joint-space and Cartesian error, respectively, but the auxiliary reference paths are identical. }\label{RR_joint}
\end{figure}
\begin{figure}
    \centering
    \includegraphics[trim=0.3cm 0.1cm 0.7cm 1.1cm, clip=true,scale=0.6, width=1.0\linewidth, height=0.59\linewidth]{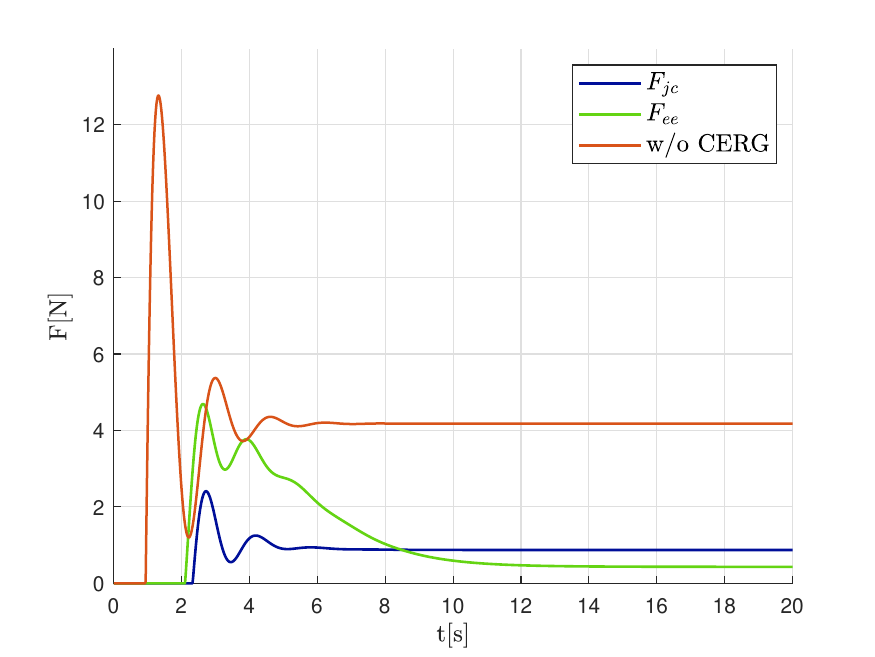}
    \caption{Forces exchanged between the RR arm and the wall for CERG with a joint space controller, CERG with end effector controller, and without CERG. \label{RR arm force}}
\end{figure}
\[
        V(x, v) = \frac{1}{2} \dot q^\top M(q) \dot q + \frac{1}{2} (f(q)-f(v))^\top K_P (f(q)-f(v)),
\]
is monotone time decreasing, but it \textbf{a)} may stagnate whenever $f(q)-f(v)\in\mathrm{null}(J(q)^\top)$, and \textbf{b)} proves the stability of the equilibrium subspace $f(q)=f(v)$ as opposed to the equilibrium point $q=v$. Interestingly, the theory presented applies  to end effector control without modifications.
Given $k_p=16$ and $k_d=10$, the comparison in Fig. \ref{RR_joint} shows that the CERG implemented in end effector space successfully reaches the configuration that minimizes the Cartesian distance between $f(q)$ and $f(v)$ and the steady-state mapping \eqref{eq:ss-map} is preserved. Meanwhile, the $\delta_s$ -$F_{ss}$  relationship becomes less intuitive in joint space.

Despite the apparent advantages of end effector control (which the CERG can perform, if desired), the joint-space formulation is preferable in a hierarchical architecture, since task-space to joint-space conversion is typically handled by the high-level planner.
Using two different controllers confirms CERG acts as a modular add-on for any low-level controller.
Fig. \ref{RR arm force} shows that CERG reduces interaction forces in both formulations as compared to operating without it.
The $E_{\max}$ bounds were respected in both cases (energy plots omitted for brevity), confirming that CERG also limits interaction forces.
\vspace{-5pt}
\subsection{Franka Emika Panda Simulation using DRAKE}\vspace{-5pt}
For a more realistic example, we simulate the 7-DoF Franka Emika using Drake's Compliant Contact Model, which accounts for dissipation, stiffness, and static/dynamic friction, \cite{drake}.
The robot is subject to hard constraints on joint torque, joint velocities, and joint limits, that are necessary for the functioning of the robot, as specified in \cite{FCI}.

\begin{figure}
    \centering
    \includegraphics[width=0.8\columnwidth]{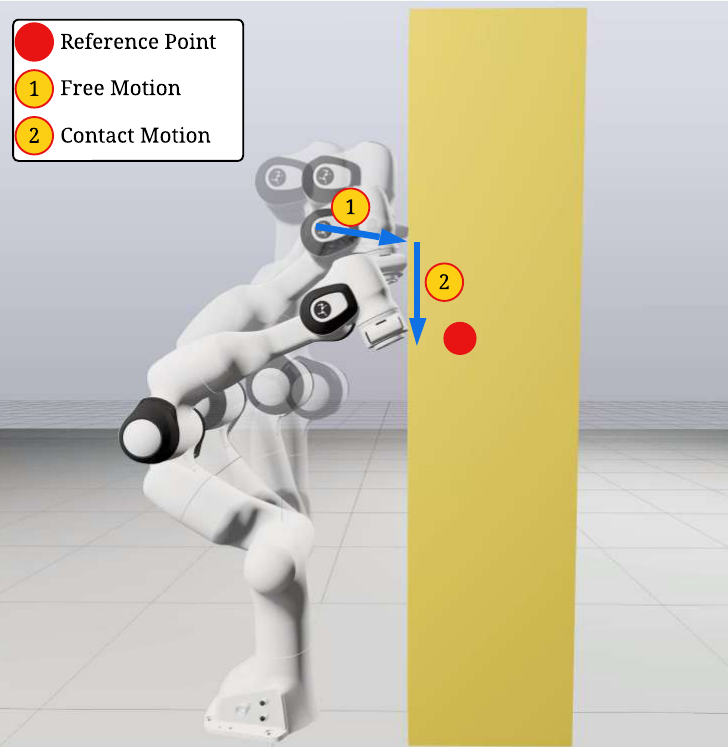}
    \caption{ Task description for CERG evaluated in DRAKE simulation with a Franka Emika Panda. The robot starts in free motion and then slides along the wall.\label{FR} }
\end{figure}

The reference, again, was chosen to force the arm to come into contact with the obstacle. In this case, the end effector constraint is $p_x\leq0.2$ and the end effector reference is $f(r)=[0.3, 0, 0.59]^\top$. Fig. \ref{FR} shows the simulation environment and the final position of the robot arm. Fig. \ref{cF} shows the energy of interaction for CERG and ERG. As expected, the CERG successfully limits the total energy of the closed-loop system.
\begin{figure}
    \centering
    \includegraphics[trim=0cm 0.2cm 0cm 0.1cm, clip=true,width=\columnwidth]{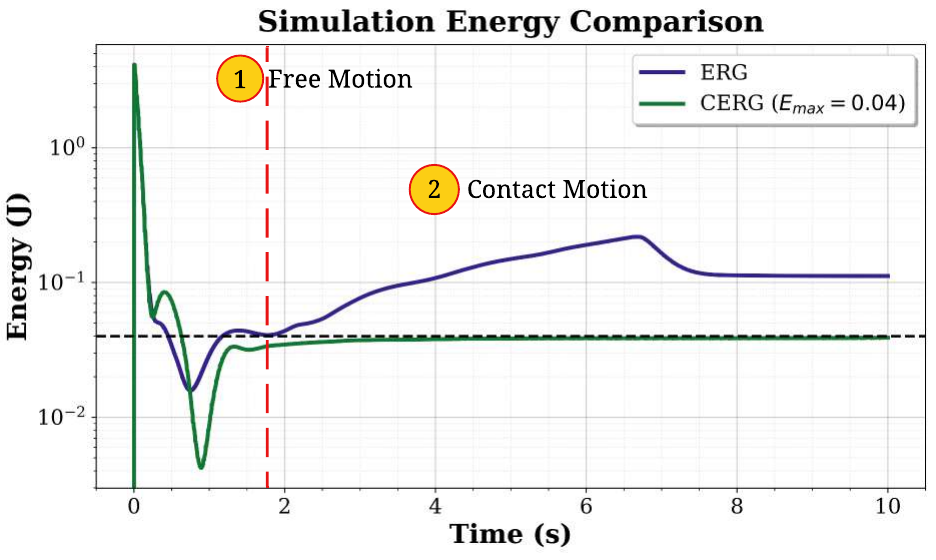}
    \caption{Energy comparison for compliant point contact model simulation in DRAKE. CERG adheres to $E_{\max}$
during contact while matching the ERG's free-motion energy profile.\label{cF} }
\end{figure}

 \begin{figure}
    \centering
    \includegraphics[trim=0cm 0.2cm 0cm 0cm, clip=true, width=\columnwidth]{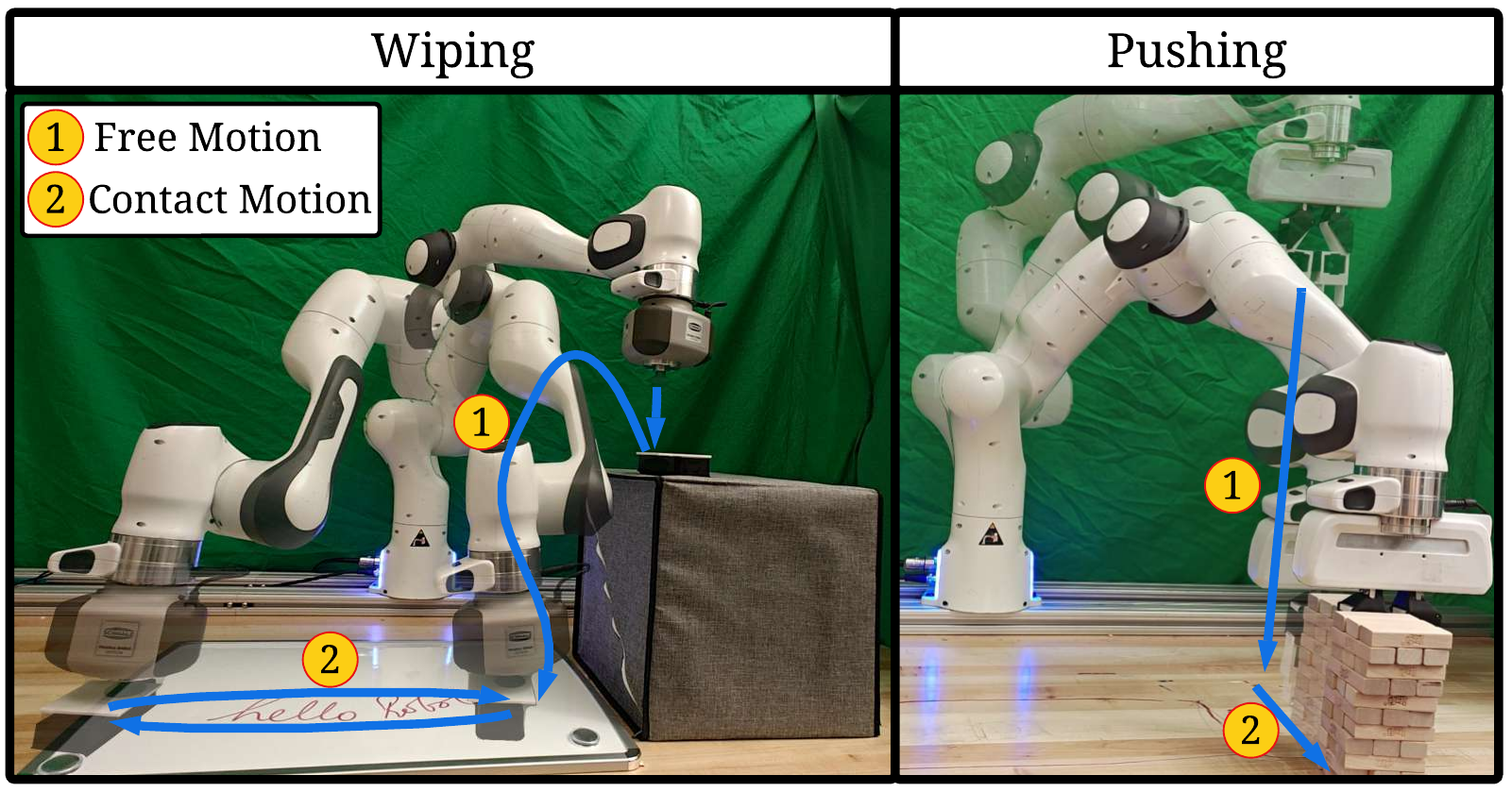}
    \caption{Experiment Setup with the Franka Emika. CERG structure is validated on 2 tasks, Wiping a Whiteboard and Pushing a Jenga tower. \label{fig:real_world}}
\end{figure}
\vspace{-10pt}

\subsection{Franka Emika Panda Real World Experiments}\vspace{-10pt}
\begin{figure}[t]
  \centering
  \includegraphics[trim=0cm 0.2cm 0cm 0.0cm, clip=true,width=\columnwidth]{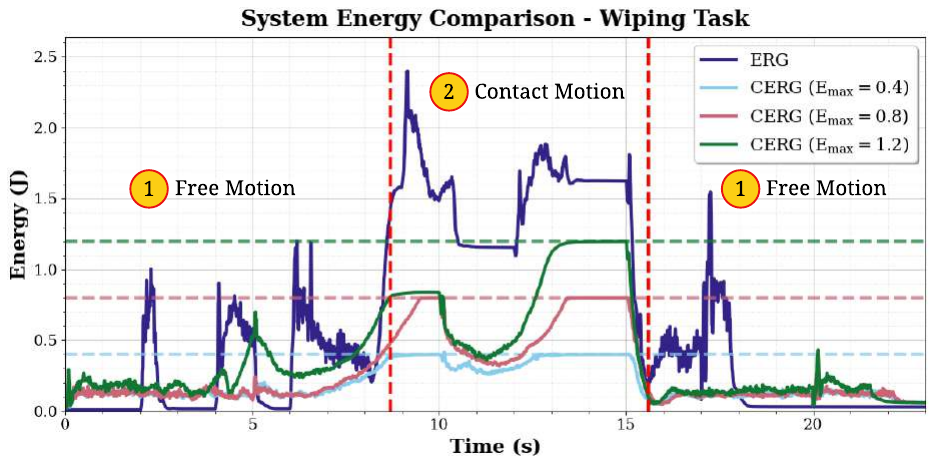}
  \caption{System energy comparison for the wiping task. The ERG violates energy limits during contact, while CERG enforces every imposed $E_{\max}$ precisely.}
  \label{fig:wiping_energy}
\end{figure}
\begin{figure}[t]
  \centering
  \includegraphics[trim=0cm 0.15cm 0cm 0cm, clip=true,width=\columnwidth]{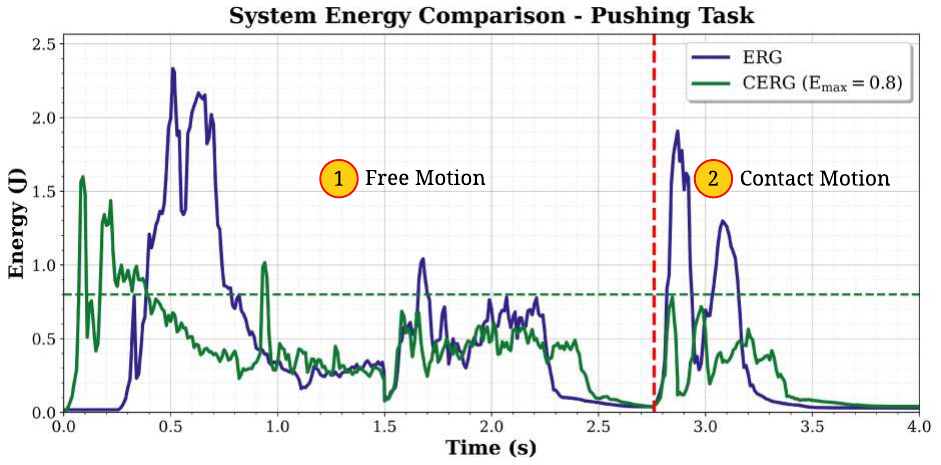}
  \caption{System energy comparison for pushing task. ERG topples the tower as it cannot enforce energy limits; CERG adheres to $E_{\max}$ and completes the task. \label{pushing_energy}}
  \label{fig:pushing_energy}
\end{figure}
\begin{table}[t]
  \centering
  \caption{Real-world experiments: Wiping and Pushing.}
  \label{tab:franka_realworld}
  \resizebox{\columnwidth}{!}{%
  \begin{tabular}{clccc}
    \toprule
    \textbf{Task} & \textbf{Controller}
      & \textbf{Task}
      & \textbf{Limited}
      & \textbf{Constr.}\\
      & & \textbf{Done} & $E_{\max}$ & \textbf{Enf.}\\
    \midrule
    \multirow{3}{*}{\shortstack[c]{Wiping \\ \simplehref{https://yaashia-g.github.io/publications/CERG/\#:~:text=and\%20impedance\%20control.-,1.\%20Wiping\%20Task,-In\%20this\%20task}{(\textcolor{blue}{Videos})}}}
      & Impedance  & \cmark & \xmark & \xmark \\
      & ERG        & \cmark & \xmark & \cmark \\
      & CERG       & \cmark & \cmark & \cmark \\
    \midrule
    \multirow{3}{*}{\shortstack[c]{Pushing \\ \simplehref{https://yaashia-g.github.io/publications/CERG/\#:~:text=2.\%20Jenga\%20Tower\%20Pushing\%20and\%20Perturbation}{(\textcolor{blue}{Videos})}}}
      & Impedance  & \xmark & \xmark & \xmark \\
      & ERG        & \xmark & \xmark & \cmark \\
      & CERG       & \cmark & \cmark & \cmark \\
    \bottomrule
  \end{tabular}%
  }
  \vspace{-5pt}
\end{table}

Having validated CERG in simulation,
we validate CERG on two real-world contact-rich tasks using the Franka Emika Panda: whiteboard wiping and Jenga pushing. The robot was subject to hard constraints, as in the previous example. The performance was measured against ERG and a Cartesian Impedance Controller with null space control, \cite{AlbuSchaeffer2003}. CERG and ERG share the same nominal joint-space controller; the impedance controller was tuned per task. The Jenga tower face towards the robot and the whiteboard were encoded as soft constraints.
Table~\ref{tab:franka_realworld} summarizes results: CERG prevents constraint violations in both tasks, while ERG and impedance control each fail in at least one.
\vspace{-8pt}
\subsubsection{Wiping} (Fig. \ref{fig:real_world}, left) The wiping task requires the robot to grasp an eraser, then sweep it back and forth across a whiteboard. Since the task is not energy-sensitive, \textcolor{blue}{\simplehref{https://yaashia-g.github.io/assets/video/Cerg\%20wiping.mp4}{CERG}}, \textcolor{blue}{\simplehref{https://yaashia-g.github.io/assets/video/Erg\%20wiping.mp4}{ERG}}, and \textcolor{blue}{\simplehref{https://yaashia-g.github.io/assets/video/wiping_impedance_ideal.webm}{Impedance control}} all complete it.
During contact, CERG enforces three different $E_{\max}$ values (Fig.\ref{fig:wiping_energy}) while the ERG does not.
The Cartesian Impedance controller, on the other hand, does not deal with hard constraints. When put in an initial configuration far away from the eraser, the controller \textcolor{blue}{\simplehref{https://yaashia-g.github.io/assets/video/impedance_jenga_bad.webm}{goes into a joint velocity violation and the motion stops}}. However, \textcolor{blue}{\simplehref{https://yaashia-g.github.io/assets/video/cerg_wiping_good.webm}{CERG is able to complete the task without any hard limit violations}}.  \vspace{-5pt}
\subsubsection{Pushing} (Fig. \ref{fig:real_world}, right) The pushing task requires the robot to push a stacked Jenga tower without toppling it. This highly energy-sensitive task demonstrates bounded energy transfer while still achieving the objective. By limiting contact energy, CERG \textcolor{blue}{\simplehref{https://yaashia-g.github.io/assets/video/Cerg\%20jenga.mp4}{pushes the tower without toppling it}}, while exceeding $E_{\max}$ in free motion with no performance penalty. \textcolor{blue}{\simplehref{https://yaashia-g.github.io/assets/video/Erg\%20jenga.mp4}{ERG fails}} due to excessive contact energy (Fig.~\ref{pushing_energy}). The impedance controller, \textcolor{blue}{\simplehref{https://yaashia-g.github.io/assets/video/Impedance\%20jenga\%20ideal.webm}{tuned for gentle pushing}}, either \textcolor{blue}{\simplehref{https://yaashia-g.github.io/assets/video/bad_impedance_jenga.webm}{violates joint velocity limits from distant starts}} or, when tuned slower, \textcolor{blue}{\simplehref{https://yaashia-g.github.io/assets/video/impedance_jenga.webm}{fails to transition smoothly and topples the tower}}.

The experiments confirm that CERG enables seamless phase transitions, disjunctive constraint satisfaction, and energy-bounded safe interaction without controller redesign.
\vspace{-8pt}

\section{Conclusion and Future Works}\vspace{-10pt}
This paper introduced the Compliant ERG, a safety filter that modifies the reference of a robot arm, ensuring satisfaction of constraints. The key innovation of the Compliant ERG is that it enables the robot to seamlessly transition between two operating modes: \emph{free motion}, where the energy of the robot is unrestricted but contact is not allowed, and \emph{contact}, where contact is allowed but the robot has a limited amount of available energy. Numerical simulations, and real world experiments provide insight on the behavior of the CERG and validate the feasibility of the approach. Future work will focus on interfacing the CERG with task-oriented planners. \vspace{-5pt}

% Shrink the font size
\footnotesize

% Reduce spacing between items (requires natbib)
\setlength{\bibsep}{0pt plus 0.3ex}
\bibliography{ifacconf}

\begin{thebibliography}{19}
\providecommand{\natexlab}[1]{#1}
\providecommand{\url}[1]{\texttt{#1}}
\providecommand{\urlprefix}{URL }
\expandafter\ifx\csname urlstyle\endcsname\relax
  \providecommand{\doi}[1]{doi:\discretionary{}{}{}#1}\else
  \providecommand{\doi}{doi:\discretionary{}{}{}\begingroup
  \urlstyle{rm}\Url}\fi

\bibitem[{Albu-Sch{\"a}ffer et~al.(2003)Albu-Sch{\"a}ffer, Ott, and
  Hirzinger}]{AlbuSchaeffer2003}
Albu-Sch{\"a}ffer, A., Ott, C., and Hirzinger, G. (2003).
\newblock Cartesian impedance control of redundant robots: Recent results.
\newblock In \emph{Proc. IEEE Int. Conf. Robot. Autom. (ICRA)}.

\bibitem[{Arimoto(1996)}]{Arimoto1996}
Arimoto, S. (1996).
\newblock \emph{Control Theory of Non-linear Mechanical Systems: A
  Passivity-based and Circuit-theoretic Approach}.
\newblock Oxford University Press.

\bibitem[{Dickson et~al.(2025)Dickson, Garcia, Anderson, Jing,
  Alizadeh-Shabdiz, Wang, DeLorey, Patterson, and Sabelhaus}]{dickson2025safe}
Dickson, A.K., Garcia, J.C.P., Anderson, M.L., Jing, R., Alizadeh-Shabdiz, S.,
  Wang, A.X., DeLorey, C., Patterson, Z.J., and Sabelhaus, A.P. (2025).
\newblock Safe autonomous environmental contact for soft robots using control
  barrier functions.
\newblock \emph{arXiv preprint arXiv:2504.14755}.

\bibitem[{Elbanhawi and Simic(2014)}]{elbanhawi2014sampling}
Elbanhawi, M. and Simic, M. (2014).
\newblock Sampling-based robot motion planning: A review.
\newblock \emph{IEEE Access}, 2, 56--77.

\bibitem[{{Franka Emika GmbH}(2025)}]{FCI}
{Franka Emika GmbH} (2025).
\newblock Robot and interface specifications.

\bibitem[{Hosseinzadeh et~al.(2019)Hosseinzadeh, Cotorruelo, Limon, and
  Garone}]{OR_constraints}
Hosseinzadeh, M., Cotorruelo, A., Limon, D., and Garone, E. (2019).
\newblock Constrained control of linear systems subject to combinations of
  intersections and unions of concave constraints.
\newblock \emph{IEEE Control Syst. Lett.}, 3(3).

\bibitem[{Huang et~al.(2024)Huang, Aydinoglu, Jin, and Posa}]{adaptive_ci_mpc}
Huang, W.C., Aydinoglu, A., Jin, W., and Posa, M. (2024).
\newblock Adaptive contact-implicit model predictive control with online
  residual learning.
\newblock In \emph{Proc. IEEE Int. Conf. Robot. Autom. (ICRA)}.

\bibitem[{Le~Cleac'h et~al.(2024)Le~Cleac'h, Howell, Yang, Lee, Zhang, Bishop,
  Schwager, and Manchester}]{Fast_CI_MPC}
Le~Cleac'h, S., Howell, T.A., Yang, S., Lee, C.Y., Zhang, J., Bishop, A.,
  Schwager, M., and Manchester, Z. (2024).
\newblock Fast contact-implicit model predictive control.
\newblock \emph{IEEE Trans. Robot.}

\bibitem[{Li et~al.(2021)Li, Figueroa, Shah, and Shah}]{li2021provably}
Li, S., Figueroa, N., Shah, A., and Shah, J. (2021).
\newblock Provably safe and efficient motion planning with uncertain human
  dynamics.
\newblock In \emph{Robotics: Science and Systems}.

\bibitem[{Liang et~al.(2023)Liang, Wang, Mao, and Yang}]{cbf_contact}
Liang, H., Wang, X., Mao, J., and Yang, J. (2023).
\newblock Control barrier function-based force constrained safety compliance
  control for manipulator.
\newblock In \emph{28th Int. Conf. Autom. Comput. (ICAC)}.

\bibitem[{Lynch and Park(2017)}]{ModernRobotics}
Lynch, K.M. and Park, F.C. (2017).
\newblock \emph{Modern robotics}.
\newblock Cambridge University Press.

\bibitem[{Merckaert et~al.(2024)Merckaert, Convens, Nicotra, and
  Vanderborght}]{RRT+ERG}
Merckaert, K., Convens, B., Nicotra, M.M., and Vanderborght, B. (2024).
\newblock Real-time constraint-based planning and control of robotic
  manipulators for safe human–robot collaboration.
\newblock \emph{Robot. Comput.-Integr. Manuf.}, 87, 102711.

\bibitem[{Merckaert et~al.(2022)Merckaert, Convens, Wu, Roncone, Nicotra, and
  Vanderborght}]{ERG_Robot}
Merckaert, K., Convens, B., Wu, C.J., Roncone, A., Nicotra, M.M., and
  Vanderborght, B. (2022).
\newblock Real-time motion control of robotic manipulators for safe
  human–robot coexistence.
\newblock \emph{Robot. Comput.-Integr. Manuf.}, 73, 102223.

\bibitem[{Michel et~al.(2022)Michel, Ott, and Lee}]{Energy_Passive_control}
Michel, Y., Ott, C., and Lee, D. (2022).
\newblock Safety-aware hierarchical passivity-based variable compliance control
  for redundant manipulators.
\newblock \emph{IEEE Trans. Robot.}, 38(6), 3899--3916.

\bibitem[{Nicotra and Garone(2018)}]{ERG}
Nicotra, M.M. and Garone, E. (2018).
\newblock The explicit reference governor: A general framework for the
  closed-form control of constrained nonlinear systems.
\newblock \emph{IEEE Control Syst. Mag.}

\bibitem[{Ortega et~al.(1998)Ortega, Loria, Nicklasson, and
  Sira-Ramirez}]{ortega1998euler}
Ortega, R., Loria, A., Nicklasson, P.J., and Sira-Ramirez, H. (1998).
\newblock \emph{{E}uler--{L}agrange systems}.
\newblock Springer.

\bibitem[{Pang et~al.(2023)Pang, Suh, Yang, and Tedrake}]{GlobalContact}
Pang, T., Suh, H.J.T., Yang, L., and Tedrake, R. (2023).
\newblock Global planning for contact-rich manipulation via local smoothing of
  quasi-dynamic contact models.
\newblock \emph{IEEE Trans. Robot.}

\bibitem[{Stüber et~al.(2020)Stüber, Zito, and Stolkin}]{PushingPaper}
Stüber, J., Zito, C., and Stolkin, R. (2020).
\newblock Let's push things forward: A survey on robot pushing.
\newblock \emph{Front. Robot. AI}, Volume 7 - 2020.

\bibitem[{Tedrake and Team(2019)}]{drake}
Tedrake, R. and Team, D.D. (2019).
\newblock Drake: Model-based design and verification for robotics.

\end{thebibliography}

\end{document}